%% file: continuous-learning.tex
\documentclass[oneside,10pt]{amsart}

\usepackage[margin=1.2in]{geometry}
\usepackage{amsmath,amsfonts,amssymb}%
\usepackage[all]{xy}
\usepackage[usenames]{xcolor}
\usepackage{setspace}
\usepackage{comment}
\usepackage{lmodern}

\usepackage{amsthm,stmaryrd}%
\usepackage{mathrsfs}

\definecolor{RedColor}{rgb}{.75,0,.25}
\definecolor{DarkBlueColor}{rgb}{.25,0.25,1}

\usepackage[colorlinks,citecolor=DarkBlueColor,urlcolor=DarkBlueColor,linkcolor=RedColor]{hyperref}

\usepackage{enumerate}

\usepackage{xparse, expl3}

\usepackage{float}
\usepackage{tikz}
\usetikzlibrary{arrows.meta,arrows}
\usetikzlibrary{decorations.pathreplacing}

\usepackage[capitalize]{cleveref}   

\newtheorem{theorem}{Theorem}[section]
\newtheorem{proposition}[theorem]{Proposition}
\newtheorem{lemma}[theorem]{Lemma}
\newtheorem{corollary}[theorem]{Corollary}
\theoremstyle{definition}
\newtheorem{definition}[theorem]{Definition}
\newtheorem{example}[theorem]{Example}
\newtheorem{remark}[theorem]{Remark}
\newtheorem{algorithm}[theorem]{Algorithm}

\crefname{theorem}{Theorem}{Theorems}
\crefname{proposition}{Proposition}{Propositions}
\crefname{lemma}{Lemma}{Lemmas}
\crefname{corollary}{Corollary}{Corollaries}
\crefname{definition}{Definition}{Definitions}
\crefname{example}{Example}{Examples}
\crefname{remark}{Remark}{Remarks}
\crefname{algorithm}{Algorithm}{Algorithms}
\crefname{equation}{Equation}{Equations}
\crefname{section}{Section}{Sections}
\crefname{subsection}{Section}{Sections}

\usepackage{microtype}
\usepackage{hyperref}
\usepackage{graphicx}

\usepackage{bbm}
\usepackage{xfrac}

\input{macros}

\begin{document}

\title[On Computable Learning of Continuous Features]{On Computable Learning of Continuous Features}

\author[N.~Ackerman]{Nathanael Ackerman}
\address{Harvard University, Cambridge, MA 02138, USA}
\email{nate@aleph0.net}

\author[J.~Asilis]{Julian Asilis}
\address{Computer Science Department, Boston College, Chestnut Hill, MA 02467, USA}
\email{julian.asilis@bc.edu}

\author[J.~Di]{Jieqi Di}
\address{Department of Mathematics, Boston College, Chestnut Hill, MA 02467, USA}
\email{dij@bc.edu}

\author[C.~Freer]{Cameron Freer}
\address{Department of Brain and Cognitive Sciences, Massachusetts Institute of Technology, Cambridge, MA 02139, USA}
\email{freer@mit.edu}

\author[J.~Tristan]{Jean-Baptiste Tristan}
\address{Computer Science Department, Boston College, Chestnut Hill, MA 02467, USA}
\email{tristanj@bc.edu} 


\begin{abstract}
We introduce definitions of \emph{computable PAC learning} for binary classification over computable metric spaces. We provide sufficient conditions for learners that are empirical risk minimizers (ERM) to be computable, and bound the strong Weihrauch degree of an ERM learner under more general conditions. We also give a presentation of a hypothesis class that does not admit any proper computable PAC learner with computable sample function, despite the underlying class being PAC learnable.
\end{abstract}


\maketitle

\tableofcontents
\addtocontents{toc}{\setcounter{tocdepth}{2}}

\section{Introduction}\label{sec:intro}

The modern statistical learning theory framework for the study
of uniform learnability is the synthesis of two theories. On the one
hand, \emph{Vapnik--Chervonenkis (VC) theory} \cite{Vapnik} is a statistical
theory that provides a rate of convergence for a uniform law of large
numbers for estimates of the form $\frac{1}{n} \sum_{i=1}^n
\mathbb{I}(f(X_i) \neq Y_i)$, where $(X_i,Y_i)$ are i.i.d.\ samples
from an unknown probability measure over $\cX \times \cY$ and $f\colon \cX \to \cY$ is a function
from a class $\HypClass$
of measurable functions.  The rate of
convergence is a function of the complexity of the class
$\HypClass$, measured using the concept of \emph{VC dimension}. On
the other hand, \emph{efficient Probably Approximately Correct (PAC) learnability} \cite{Valiant} is a computational
theory that defines the efficient learnability of a function class
$\HypClass$ in terms of the existence of a \emph{learner}, given by an algorithm
having polynomial runtime, that takes an i.i.d.\ sample
$S = \bigl((X_i,Y_i)\bigr)_{i < n}$
from an unknown probability measure $\mu$ as input
and returns a function $h\in\HypClass$ whose error $\Pr(h(X) \neq Y)$
for $(X,Y) \sim \mu$
can be bounded with high probability over the choice of $S$.
The analogous notion of \emph{PAC learnability}, where the learner is merely required to be \emph{measurable} in an appropriate sense, rather than efficiently computable, has also been widely studied.

The synthesis of these two theories culminates with the so-called
fundamental theorem of machine learning \cite{Blumer}, which
establishes, under certain broadly-applicable measurability conditions,
that a class of functions is PAC learnable if and only if
its VC dimension is finite. This theory provides a justification for
the foundational learning paradigm of empirical risk minimization and
has become the basis for studying many other learning paradigms and
non-uniform theories of learnability. Note, however, that in this
framework the learner is only required to be a measurable function, and in particular need not be computable.

Insofar as the goal of studying uniform learning is to determine when a problem admits
supervised learning by some program given access to training examples, it is important
to investigate the subclass of learners that are in some sense \emph{computable},
a natural object of study intermediate between learners that are efficiently computable and
those that are merely measurable.
In this direction, 
\cite{ALT} proposed a notion of computable learner for computably represented
hypothesis classes $\HypClass$ on discrete spaces.
They principally consider binary classification in the case where 
$\HypClass$ is a computably enumerable set of computable functions
on a countable domain, e.g., $\cX = \N$.

However, many natural problems considered in classical PAC learning theory
have continuous domains, such as $\R^n$.
In the present paper, we consider notions of computable learners and hypothesis classes,
without restricting to the discrete setting, e.g., where $\cX$ is an arbitrary computable metric space.
We do so using the framework of computable
analysis \cite{Weihrauch}, and establish upper and lower bounds on the computability of
several standard classes of learners in our setting.

We now describe the structure of the paper.
Next, in \cref{subsec:related}, we describe several other approaches to computability in learning theory, including \cite{ALT}, and their relation to our work.
We then in \cref{sec:prelim} provide the relevant preliminaries from computability theory (including computable metric spaces and Weihrauch reducibility) and from classical PAC learning theory.
In \cref{sec:notions} we develop the basic concepts of computable learning theory in our setting, including notions of computability for learners, presentations of hypothesis classes, and sample functions.
Section~\ref{sec:bounds} contains our primary results,
including sufficient conditions for empirical risk minimizer (ERM) learners to be computable,
upper bounds on the strong Weihrauch degrees of certain ERM learners,
and the construction of a (computable presentation of a) hypothesis class that is PAC learnable but which has no computable proper PAC learner that admits a computable sample function.

\subsection{Related work}
\label{subsec:related}

Computability of PAC learners has also been studied in \cite{ALT}, which considers 
the setting of \emph{discrete} features and \emph{countable} hypothesis classes.
They provide several positive and negative results on the computability of both proper and improper learners for various notions of computably presented hypothesis classes, in both the realizable and agnostic cases.
Our results, when we restrict our setting to discrete spaces, correspond most closely to their results for so-called \emph{recursively enumerably representable} (RER) hypothesis classes. 
In particular, our \cref{thm:realizable ERM}
can be viewed as a generalization of \cite[Theorem~10]{ALT}, and
the proof of our \cref{thm:discrete-lower} uses similar ideas to those in \cite[Theorem~11]{ALT}.

Computability of \emph{non-uniform} learning, which we do not consider in this paper, has been studied in the discrete setting in both \cite{Soloveichik} and \cite{ALT}.

In the present paper (and \cite{ALT}) when considering a function with finite codomain (as arises for both learners and presentations of hypothesis classes), the notion of computable function is such that for each input, the output is always eventually given. It is also reasonable to consider settings in which there is a particular value signaling non-halting, which the computable function may never identify. 
This approach is explored in \cite{Crook}, where non-halting of a learner's output is signaled by the value $\bot$.
A related approach is considered in \cite{Calvert}, which studies PAC learning for concepts that are $\Pi^0_1$ classes on $\Cantor$, which can be thought of as equivalent to working with computable functions from $\Cantor$ to Sierpi\'nski space $\Sier$ (i.e., the space $\{\bot, \top\}$ with open sets $\{\emptyset, \{\top\}, \{\bot, \top\}\}$), where the inverse image of $\top$ is the $\Pi^0_1$ class in question.

Another interaction between learning theory and computability is in the setting of ``learning in the limit'' \cite{Gold}, sometimes called \emph{TxtEx learning}. One recent result \cite{Beros} in this framework establishes the $\Sigma^0_3$-completeness of this learning problem for certain computably enumerable hypothesis classes.

\section{Preliminaries}\label{sec:prelim}

This section provides a brief treatment of the computability theory and classical learning theory that form the starting point of our study.

We begin by recalling several pieces of notation.
For a set $I$, we write $(s_i)_{i \in I}$ to denote an $I$-indexed sequence.
For $n\in\N$, write $[n]$ to denote the set $\{0, 1, \ldots, n-1\}$.
We write $f\restricted{U}$ to denote the restriction of a function $f\colon X\to Y$ to a subdomain $U\subseteq X$.

For a topological space $\cX$, we write $\cX^{<\omega}$ for the space 
$\coprod_{i \in \N} \cX^i$
of finite sequences of points in $\cX$, endowed with its natural topology as the coproduct of product spaces.
An \defn{extended metric space} is a set $X$ equipped with a distance function $d \colon X \times X \to \R \cup \{\infty\}$ satisfying the usual metric axioms (where $\infty + r = \infty$ for any $r\in \R\cup\{\infty\}$).

\subsection{Computable metric spaces and Weihrauch reducibility}

We next describe certain key notions of computability and computable analysis, including the notions of computable metric spaces and computable functions between them.
For more details and several equivalent formulations of the basic notions, see, e.g., \cite[Section~4]{MR2762094}. We then describe the notion of Weihrauch reducibility; for more details, see \cite{BGPsurvey}.

Recall that a partial function $f$ from $\N$ to $\N$ is said to be \defn{computable} if there is some Turing machine that halts on input $n$ (encoded in binary) precisely when $f$ is defined on $n$, and in this case produces (a binary encoding of) $f(n)$ as output. 
We fix a standard encoding of Turing machines and write $\{e\}$ to denote the partial function that the program encoded by $e\in\N$ represents.  We write $\{e\}(n)\halts$ to mean that the partial function $\{e\}$ is defined on $n$, i.e., that the program encoded by $e$ halts on input $n$, and write $\{e\}(n)\nohalts$ otherwise.

In this paper, it will be convenient to take oracles to be elements of $\Baire$ rather than $\Cantor$.
For $f\in\Baire$ we write $\{e\}^f$ to denote the partial function defined by an oracle program encoded by $e$ using $f$ as an oracle.
Because we are using oracles in $\Baire$, we will define the Turing jump to yield a function rather than a set.  Given $f\in\Baire$, the \defn{Turing jump} of $f$, written $f'$, is defined to be the characteristic function of $\{ e\in\N \st \{e\}^f(0)\halts\}$. By convention, we write $\HaltingSet$ for the characteristic function of the halting set $\{ e \in \N \st \{e\}(0)\halts\}$.

A subset of $\N$ is \defn{computable} if its characteristic function is a total computable function, and is \defn{computably enumerable} (c.e.) if it is the domain of a partial computable function (equivalently, either empty or the range of a total computable function). We will also speak of more elaborate finitary objects (such as sets of finite tuples of rationals) as being computable or c.e.\ when they are computable or c.e., respectively, under a standard encoding of the objects via natural numbers.

For concreteness, we will use the notion of a \emph{presentation} of a real when defining computable metric spaces, but note that this could also be formulated using represented spaces, as defined later in the section.
An \defn{extended real} is an element of $\R \cup \{\infty\}$.
A \defn{presentation} of an extended real
is a sequence of rationals $(q_i)_{i \in \N}$ with either $q_i > i$ for all $i$ or $|q_i - q_j | < 2^{-i}$ for all $i$, $j$ with $i < j$.  In the first case we say that the sequence is a presentation of $\infty$ and in the second case that it is a presentation of the limit of the Cauchy sequence in $\R$.
We say that an extended real is \defn{computable} if it has a computable presentation. The \defn{computable reals} are the elements of $\R$ admitting a computable presentation as extended reals.

We say that a sequence $(t_i)_{i\in\N}$ in an (extended) metric space $\cX = (X, d)$ is a \defn{rapidly converging Cauchy sequence} when for all $i < j$ we have $d(t_i, t_j) < 2^{-i}$.

\begin{definition}
A \defn{computable (extended) metric space} is a triple $\bbX = (X, d_{\bbX}, (s^{\bbX}_i)_{i\in\N})$ such that
  \begin{enumerate}
    \item $(X \cup S,  d_{\bbX})$ is a separable (extended) metric space, where $S=\{s^{\bbX} \st i\in\N\}$,
    \item $(s^{\bbX}_i)_{i \in \N}$, called the sequence of \defn{ideal points of $\bbX$}, enumerates a dense subset of $(X  \cup S, d_{\bbX})$, 
\item $X$, called the \defn{underlying set} of $\bbX$, is dense in $(X \cup S, d_{\bbX})$, and
    \item $d_{\bbX}$, called the \defn{distance function}, is such that $d_{\bbX}(s^{\bbX}_i, s^{\bbX}_j)$ is a computable extended real, uniformly in $i$ and $j$.
  \end{enumerate}
In the special case where $(X, d_{\bbX})$ is a \emph{complete} (extended) metric space, we say that $\bbX$ is a \defn{computable (extended) Polish space}.
An element $x \in X$ is said to be a \defn{computable point} of $\bbX$ if there is a computable function $f\colon \N \to \N$ such that $(s^\bbX_{f(i)})_{i\in\N}$ is a rapidly converging Cauchy sequence that converges to $x$. We will omit the superscripts and subscripts when they are clear from context.
\end{definition}

Note that some papers (e.g., \cite[Definition~2.1]{MR2013565} and \cite[Definition~7.1]{MR2762094}) define a computable metric space only in the case where the set $S$ of ideal points is a subset of $X$, and others (e.g., \cite[Definition~2.4.1]{HoyrupRojas}) use the term computable metric space to refer to what we call a computable Polish space.

\begin{example}
The set $\R$ of real numbers forms a computable Polish space under the Euclidean metric, when equipped with the set $\Q$ of rationals as ideal points under the standard diagonal enumeration $(q_i)_{i\in\N}$. The computable points of this computable Polish space are precisely the computable reals.  
\end{example}

Note that in a computable metric space that is not a Polish space, the ideal points need not be in the underlying set, as in the following example.

\begin{example}
The set of irrational numbers forms a computable metric space under the Euclidean metric, again equipped $(q_i)_{i\in\N}$ as the sequence of ideal points. The computable points of this computable metric space are the computable irrational numbers.
\end{example}

The next two examples will be key in many of our constructions.

\begin{example}
\emph{Baire space}, written $\Baire$, is the computable Polish space consisting of countably infinite sequences of natural numbers,
with ideal points those sequences having only finitely many nonzero values (ordered lexicographically), and where $d_{\Baire}$ is the ultrametric on the countably infinite product of $\{0, 1\}$, i.e., 
\[
d_{\Baire}\bigl( (s_i)_{i\in\N}, (t_i)_{i\in\N} \bigr) = 
2^{- \inf_{i\in\N} (s_i \neq t_i)}.
\]
\emph{Cantor space}, written $\Cantor$, is the computable Polish subspace of $\Baire$ consisting of binary sequences. 
\end{example}

Let $\pi_0$ and $\pi_1$ be computable maps from $\N$ to $\N$ such that $i \mapsto (\pi_0(i), \pi_1(i))$ is a computable bijection of $\N$ with $\N \times \N$.

When $\bbX$ and $\bbY$ are computable (extended) metric spaces, we write $\bbX \times \bbY$ to denote the computable (extended) metric spaces with underlying set $X \times Y$, with sequence of ideal points $\bigl( (s_{\pi_0(i)}^\bbX, s_{\pi_1(i)}^\bbY)\bigr)_{i\in\N}$, and where $\bigl( (X \cup S^{\bbX}) \times (Y \cup S^{\bbY}), d_{\bbX \times \bbY}\bigr)$ is the product (extended) metric space of $\bigl(X \cup S^{\bbX}, d_{\bbX}\bigr)$ and $\bigl(Y \cup S^{\bbY}, d_{\bbY}\bigr)$.

We let $\bbX^{<\omega}$ be the coproduct $\coprod_{n \in \omega} \prod_{i \in [n]} \bbX$, i.e., the space whose underlying set consists of finite sequences of elements of $X$, whose ideal points are finite sequences of ideal points in $X$, and where the distance function satisfies
\[
d_{\bbX^{<\omega}}\bigl((x_i)_{i \in [n]}, (y_i)_{i \in [m]}\bigr) = 
\begin{cases}
\max_{i \in [n]} \, d_{\bbX}(x_i, y_i) & \text{if~} m = n;\\
\infty &\text{otherwise}.
\end{cases}
\]

\begin{definition}
Suppose $\cX = (X, d_\cX)$ and $\cY = (Y, d_\cY)$ are metric spaces and $Z \subseteq X$. We say a map $f\colon X \to Y$ is \defn{continuous} on $Z$ if for all open sets $U \subseteq Y$, there is an open set $V \subseteq X$ such that $f^{-1}(U) \cap Z = V \cap Z$. In other words, $f$ restricted to $Z$ is continuous as a map from the metric space that $\cX$ induces on $Z$ to $\cY$.
\end{definition}

\begin{definition}
\label{def:comp-map}
Let $\bbX$ and $\bbY$ be computable metric spaces with ideal points $(s_i)_{i \in \N}$ and $(t_i)_{i \in \N}$ respectively, and suppose $Z \subseteq X$.
Suppose $f\colon W \to Y$ is a map where $Z \subseteq W \subseteq X$.
We say that $f$ is 
\defn{computable on $Z$} if for all
$(j, q) \in \N \times \Q$ there is a set $\Phi_{j, q} \subseteq \N \times \Q$ such that

\begin{itemize}
\item $f^{-1}(B(t_j, q)) \cap Z = \bigl(\bigcup_{(k, p) \in \Phi_{j, q}} B(s_{k}, p)\bigr) \cap Z$, and

\item the set $\{(j, q, k, p) \st (k, p) \in \Phi_{j, q}\}$ is c.e.
\end{itemize}
\end{definition}

This definition captures the notion that the partial map $f$ is continuous on its restriction to $Z$ and has a computable witness to this continuity.

Observe that a computable function from $\Baire$ to a computable metric space $\bbY$ can be thought of as a program on an oracle Turing machine that takes the input on its oracle tape, and outputs a ``representation'' of a point in $\bbY$. The notion of a \emph{represented space} is one way of making this notion precise.
For more details, see \cite{BGPsurvey}.

\begin{definition}
A \defn{represented space} $(X, \gamma)$ is a set $X$ along with a surjection $\gamma$
from a subset of $\N^\N$ onto $X$.
When the choice of $\gamma$ is clear from context, we call $\gamma$ the \defn{representation} of $X$.
\end{definition}

\begin{definition}
Suppose $\bbX = (X, d_\bbX, (s_i^\bbX)_{i\in\N})$
is a computable metric space. 
Define $\CauchySeq_\bbX \subseteq \N^\N$ to be the collection of functions $f\colon \N \to \N$ for which
$\bigl(s^\bbX_{f(i)} \bigr)_{i\in\N}$ is a rapidly converging Cauchy sequence whose limit is in $X$. The \defn{represented space induced by $\bbX$} is defined to be $(X,\gamma_\bbX)$, where
\[
\gamma_\bbX\colon \CauchySeq_\bbX \to X
\]
assigns each function $f$ the value $\lim_{i\to\infty} s^\bbX_{f(i)}$.
\end{definition}

Intuitively, a  \emph{realizer} of a function $g$ takes a description of an input $x$ to a description of the corresponding output $g(x)$, where these descriptions are given in terms of representations.

\begin{definition}
Suppose $(X, \gamma_X)$ and $(Y, \gamma_Y)$ are represented spaces, and let $g\colon X \to Y$ be a map.
A \defn{realizer} of $g$ is any function $G \colon \domain(\gamma_X) \to \domain(\gamma_Y)$ such that 
$\gamma_Y \circ G = g \circ \gamma_X$.

A realizer is \defn{computable} if it is computable on $\domain(\gamma_X)$ (considered as a partial map between computable metric spaces $\Baire$ and $\Baire$).
\end{definition}

The notion of \emph{strong Weihrauch reducibility} aims to capture the intuitive idea that one function is computable given the other function as an oracle, along with possibly some computable pre-processing and post-processing, where access to the original input is permitted only in pre-processing.  (The weaker notion of \emph{Weihrauch reducibility}, in which the input may be used again in post-processing, also arises in computable analysis, but in this paper we are able to show that all of the relevant reductions are strong.)

\begin{definition}
	\label{Weihrauch-defone}
Let $(X_i, \gamma_{X_i})$ and $(Y_i, \gamma_{Y_i})$ be represented spaces for $i \in \{0, 1\}$, and suppose that 
$f \colon X_0 \to Y_0$ and $g \colon X_1 \to Y_1$ are functions.
Let $\cF$ and $\cG$ be the sets of realizers of $f$ and $g$ respectively.
We say that $f$ is 
\defn{strongly Weihrauch reducible} to $g$, and write
$f \lesW g$, when 
there are computable functions $H$ and $K$, each from some subset of $\Baire$ to $\Baire$,
such that
for every $G \in \cG$ there exists an $F \in \cF$ satisfying
\[
F = H \circ G \circ K.
\]
We say that $f$ and $g$ are \defn{strongly Weihrauch equivalent}, and write $f \eqsW g$, when $f \lesW g$ and $g \lesW f$.
\end{definition}

Note that strong Weihrauch reducibility is usually described in the more general setting of partial multifunctions. Here we will only need single-valued functions with explicitly defined domains, and \cref{Weihrauch-defone} coincides with the standard one in this situation.

The following important map describes the problem
of computing limits on a represented space $X$ induced by a computable metric space $\bbX$.
(Note that elsewhere in the literature, $\lim_{\bbX}$ is typically referred to as $\lim_X$.)

\begin{definition}
Suppose $\bbX$ is a computable metric space, and let
$(X, \gamma_{\bbX})$ be the represented space it induces.
The \defn{limit map}
$\lim_{\bbX}\colon X^\N \to X$ 
is the function 
that assigns every convergent Cauchy sequence in $X$ its limit.
\end{definition}

One can view $\lim_{\Baire}$
as playing a role in Weihrauch reducibility analogous to the role played by
the halting problem $\HaltingSet$ with respect to Turing reducibility. 
For more details, see \cite[\S11.6]{BGPsurvey}.

It will also be useful to introduce the notion of a rich space, which bears a relation to $\lim_{\Baire}$ and is informally a space that computably contains the real numbers. 

\begin{definition}
A computable metric space $\bbX$ is \defn{rich} if there is some computable map
$\iota\colon \Cantor \to \bbX$
that is injective and whose partial inverse $\iota^{-1}$ is also injective.
\end{definition}

\begin{lemma}{\cite[Proposition~11.6.2]{BGPsurvey}}
\label{lem:limBaire is maximal}
If $\bbX$ and $\bbY$ are rich spaces, then $\lim_{\bbX} \eqsW \lim_{\bbY}$. In particular, $\lim_{\bbX} \eqsW \lim_{\Baire}$.
\end{lemma}

Observe that for any computable metric space $\bbX$, the space $\bbX \coprod \Baire$ is rich, and therefore $\lim_{\bbX} \lesW \lim_{\bbX\coprod \Baire} \eqsW \lim_{\Baire}$. Hence $\lim_{\Baire}$ is maximal (under $\lesW$) among limit operators. 

We will also work with the \emph{Turing jump} map $\J \colon \Baire \to \Baire$, given by $z \mapsto z'$, which is strongly Weihrauch equivalent to $\lim_{\Baire}$.

\begin{lemma}[{\cite[Theorem~11.6.7]{BGPsurvey}}]
\label{limBaire eqsW J}
$\lim_{\Baire} \eqsW \J$.
\end{lemma}

Although $\lim_{\Baire} \eqsW \J$, in general $\lim_{\cI}$ is weaker. In \cref{subsec:upper}
we will establish our upper bounds in terms of $\lim_{\cI}$ for appropriate computable metric spaces $\cI$, while in \cref{subsec:lower} we will establish a bound using the operator $\J$.

Strong Weihrauch reductions to the \emph{parallelization} of a function allow one to ask for countably many instances of the function to be evaluated.

\begin{definition}
Let $f\colon X \to Y$ be a map between represented spaces. The \defn{parallelization} of $f$ is the map $\parallelization{f}\colon X^\N \to Y^\N$ defined by 
$\parallelization{f}\bigl( (x_i)_{i\in\N}\bigr)
= \bigl( f(x_i) \bigr)_{i\in\N}$.
\end{definition}

The following is immediate.
\begin{lemma}
\label{lem:reducible to parallelization}
For any map $f \colon X \to Y$ between represented spaces, 
$f \lesW \parallelization{f}$.
\end{lemma}

We will also need the following standard fact.

\begin{lemma}[{\cite[Theorem~11.6.6]{BGPsurvey}}]
\label{lem:limBaire is parallelizable}
$\parallelization{\lim_{\Baire}} \eqsW \lim_{\Baire}$.
\end{lemma}

The notion of the parallelization of a function will be important in \cref{subsec:lower}, for reasons we explain in \cref{rem:parallel explanation}.

\subsection{Learning theory}\label{subsec:learning}

We now consider the traditional framework for uniform learnability, formulated for Borel measurable hypotheses. A learning problem is determined by a domain, label set, and hypothesis class, as we now describe.
\begin{enumerate}
\item[(i)] a \defn{domain} $\cX$ of features that is a Borel subset of some extended Polish space $\bbX$,
\item[(ii)] a \defn{label set} $\cY$ that is a Polish space, and
\item[(iii)] a \defn{hypothesis class} $\HypClass$ consisting of Borel functions from $\cX$ to $\cY$.
\end{enumerate}
We will say that any Borel function from $\cX$ to $\cY$ is a \defn{hypothesis}; note that such a map is sometimes also called a \defn{predictor}, \defn{classifier}, or \defn{concept}. In this paper, we will only consider problems in binary classification, i.e., where $\cY = \{0, 1\}$, considered as a metric space under the discrete topology.

Let $\cD$ be a Borel measure on $\cX \times \cY$.
The \defn{true error}, or simply \defn{error}, of a hypothesis $h\in\HypClass$ with respect to $\cD$ is the probability that $(x, h(x))$ disagrees with a randomly selected pair drawn from $\cD$, i.e.,
\[ L_{\cD}(h) = \cD\bigl(\{(x,y) \in \cX \times \cY\; | \; y \neq h(x)\}\bigr) . \]
The \defn{empirical error} of a hypothesis $h$ on a tuple $S = ((x_1, y_1), \dots, (x_n, y_n)) \in (\cX \times \cY)^n$ of \defn{training examples} is the fraction of pairs in $S$ on which $h$ misclassifies the label of a feature, i.e.,
\[ L_S(h) = \frac{\sum_{i=1}^n |h(x_i) - y_i| }{n}. \]

Traditionally, one thinks of a learner as a map which takes finite sequences of $(\cX \times \cY)^{<\omega}$ and returns a hypothesis, i.e., an element of $\cY^{\cX}$. We would then like to define a computable learner as a learner which is computable as a map between computable extended metric spaces. Unfortunately, here we encounter the obstructions that $\cY^{\cX}$ is not, in general, an extended metric space. We overcome it by instead considering a learner as the ``curried'' version of a map from $(\cX \times \cY)^{<\omega}$ to $\cY^{\cX}$, i.e., as a map from $(\cX  \times \cY)^{<\omega} \times \cX \to \cY$. In this manner, we will be able to consider learners which are computable as maps between Polish spaces.

\pagebreak
\begin{definition}\label{def:learner}
A \defn{learner} is a Borel measurable function $A \colon(\cX \times \cY)^{<\omega} \times \cX \to \cY$. For notational convenience, for $S \in (\cX \times \cY)^{<\omega}$ we let $\curry{A}(S)\colon \cX \to \cY$ be the function defined by $\curry{A}(S)(x) = A(S, x)$. 
\end{definition}

The goal of a learner $A$ is to
return a hypothesis $h$ that minimizes the true error with respect to an unknown Borel distribution $\cD$ on $\cX \times \cY$.
The learner does so by examining a $\cD$-i.i.d.\ sequence $S = \big((x_1, y_1), \ldots, (x_n, y_n)\big)$.
Notably, the learner cannot directly evaluate $L_{\cD}$;
it is guided only by the information contained in the sample $S$, including evaluations of $L_S$.
However, as it is ignorant of $\cD$, the learner
does not know how faithfully $L_S$ approximates $L_{\cD}$.

The most central framework for assessing learners with respect to hypothesis classes is that of PAC learning (see, e.g., \cite[Chapter~3]{UML}).
In the setting of \emph{efficient} PAC learning \cite[Definition~8.1]{UML}, one further requires that the learning algorithm be polynomial-time in the reciprocal of its inputs $\epsilon$ and $\delta$, to be described in the following definition.

\begin{definition}\label{Defi:PAC}
Let $\DD$ be a collection of Borel distributions on $\cX \times \cY$ and
let $\HypClass$ be a hypothesis class.
A learner $A$ is said to \defn{PAC learn $\HypClass$ with respect to $\DD$}  (or is a \emph{learner for $\HypClass$ with respect to $\DD$})
if there exists a 
function $m\colon(0, 1)^2 \to \N$, called a \defn{sample function}, that is non-increasing on each coordinate and satisfies the following property: for every $\epsilon, \delta \in (0, 1)$ and every Borel distribution $\cD\in\DD$,
a finite i.i.d.\ sample $S$ from $\cD$ with $|S| \geq m(\epsilon, \delta)$ is such that, with probability at least
$(1 - \delta)$ over the choice of $S$, the learner $A$ outputs a hypothesis $\curry{A}(S)$ with
\[
L_{\cD}(\curry{A}(S)) \leq \inf_{h \in \HypClass} L_{\cD}(h) + \epsilon.
\label{eq:learnable}
\tag{$\dagger$}
\]
(Observe that \eqref{eq:learnable} is a Borel measurable condition, as 
$L_{\cD}(\curry{A}(S)) 
=
\int \Indicator{A(S, x) \ne y} \cD(dx , dy)$.)
The minimal such sample function for $A$ is its \defn{sample complexity}.
When there is some learner $A$ that learns $\HypClass$ with respect to $\DD$, we say that $\HypClass$ is \defn{PAC learnable with respect to $\DD$} (via $A$).

In the case where $\DD$ consists of all Borel distributions on $\cX \times \cY$, we say
that $\HypClass$ is \defn{agnostic PAC learnable} and that $A$ is an \defn{agnostic PAC learner for $\HypClass$}.
In the case where $\DD$ consists of the class of Borel distributions $\cD$ on $\cX \times \cY$ for which $L_{\cD}(h) = 0$ for some $h \in \HypClass$,  we say
that $\HypClass$ is \defn{PAC learnable in the realizable case} and that $A$ \defn{PAC learns $\HypClass$ in the realizable case}.
\end{definition}

\begin{remark}
Some sources use ``sample complexity'' to refer to a property of hypothesis classes $\HypClass$, defined as the pointwise minimum of all of $\HypClass$'s PAC learners' sample complexities (in the sense of Definition~\ref{Defi:PAC}). The learner-dependent definition will be more appropriate for our purposes, in which, for instance, the distinction between computable and noncomputable learners is of central importance.
\end{remark}

We will see shortly in Theorem~\ref{thm:fundamental} that a class that is PAC learnable in the realizable case must also be agnostic PAC learnable (possibly via a different learner with worse sample complexity).

\begin{definition}
A learner $E$ is an \defn{empirical risk minimizer} (or \emph{ERM}) for $\HypClass$, if for all finite sequences
$S \in  (\cX \times \cY)^{<\omega}$,  we have
\[\curry{E}(S) \in \argmin_{h\in\HypClass} L_S(h).
\]
\end{definition}

\begin{definition}
The \defn{VC dimension} of $\HypClass$ is 
\[
\sup \ \bigl\{\,|C| \!\st\! C \subseteq \cX \hspace*{5pt} \text{~and~} \hspace*{5pt}
\{h \restricted{C} \!\st\! h \in \HypClass\bigr\} = \{0, 1\}^C   \}.
\]
When $\{h \restricted{C} \!\st\! h \in \HypClass\bigr\} = \{0, 1\}^C$, we say that $\HypClass$ \defn{shatters} the set $C$.
\end{definition}

We now state the relevant portions of the fundamental theorem of learning theory in our setting (binary classification with $0$-$1$ loss), which holds for hypothesis classes satisfying the mild technical assumption of \emph{universal separability} \cite[Appendix~A]{Blumer}. This condition is satisfied for any hypothesis class having a computable presentation (see \cref{Defi:Hypothesis class}), as is the case for all hypothesis classes considered in this paper.

\pagebreak
\begin{theorem}[{\cite[Theorem~6.7]{UML}}]\label{thm:fundamental}
Let $\HypClass$ be a hypothesis class of functions from a domain $\cX$ to $\{0, 1\}$. Then the following are equivalent:
\begin{enumerate}
\item[1.] $\HypClass$ has finite VC dimension.
\item[2.] $\HypClass$ is PAC learnable in the realizable case.
\item[3.] $\HypClass$ is agnostically PAC learnable.
\item[4.] Any ERM learner is a PAC learner for $\HypClass$, over any family of measures.
\end{enumerate}
\end{theorem}

Because of the equivalence between conditions 2 and 3, we will say that a hypothesis class $\HypClass$ is \emph{PAC learnable} (without reference to a class of distributions $\DD$, and without mentioning agnostic learning or realizability) when any of these equivalent conditions hold.
Note that while every agnostic PAC learner for $\HypClass$ is in particular a PAC learner for $\HypClass$ in the realizable case, the converse is not true; when we speak of a \emph{PAC learner for $\HypClass$} without mention of $\DD$, we will mean the strongest such instance, namely that it is an agnostic PAC learner for $\HypClass$.

Furthermore, there exists a connection between the VC dimension of a PAC learnable class and the sample functions of its ERM learners. 

\begin{theorem}[{\cite[pp. 392]{UML}}]\label{thm:ERM sample function}
Let $\HypClass$ be a hypothesis class of functions from a domain $\cX$ to $\{0, 1\}$ with finite VC dimension $d$. Then its ERM learners are PAC learners with sample functions
\[ m(\epsilon, \delta) = 4 \frac{32d}{\epsilon^2} \cdot \log \bigg(\frac{64d}{\epsilon^2}\bigg) + \frac{8}{\epsilon^2} \cdot (8d \log(\epsilon / d) + 2 \log(4 / \delta)). \]  
\end{theorem}

\section{Notions of computable learning theory}\label{sec:notions}

As described before \cref{def:learner}, the notion of learner we consider in this paper is the \emph{curried} version of the standard one, in order to allow for it to be a computable map between Polish spaces. We now make use of this, to define when a learner is computable, and when a hypothesis class is computably PAC learnable.

\begin{definition}\label{Defi:Computable learner}
By a \defn{computable learner} we mean a learner $A\colon(\cX \times \cY)^{<\omega} \times \cX \to \cY$ which is computable as a map of computable extended metric spaces. We say a hypothesis class $\HypClass$ is \defn{computably PAC learnable} if there is a computable learner that PAC learns it. 
\end{definition}

It will also be important to have a computable handle on hypothesis classes themselves. As such, we will primarily consider hypothesis classes as collection of hypotheses endowed with (not necessarily unique) indices. This information is collected up into a \emph{presentation} of the class.

\begin{definition}\label{Defi:Hypothesis class}
A \defn{presentation of a hypothesis class} is a Borel measurable function $\PresentationHypClass\colon\cI \times \cX \to \cY$.
We call $\cI$ the \defn{index space}.
Let $\curry{\PresentationHypClass}\colon\cI \to \cY^\cX$ be the function defined by $\curry{\PresentationHypClass}(i)(x) = \PresentationHypClass(i, x)$. 
We write $\UnderlyingHypClass$ to denote the underlying hypothesis class, i.e., $\range(\curry{\PresentationHypClass})$.
We say that $\PresentationHypClass$ \defn{presents} the class $\UnderlyingHypClass$
and that a hypothesis is an \defn{element of $\PresentationHypClass$} when it is in $\UnderlyingHypClass$.
\end{definition}

\begin{definition}\label{Defi:Computable hypothesis class}
A presentation $\PresentationHypClass\colon\cI \times \cX \to \cY$ 
of a hypothesis class 
is \defn{computable} if $\cI$ is a computable metric space and $\PresentationHypClass$ is computable as a map of computable extended metric spaces. 
\end{definition}

Classically, a \emph{proper learner} for a hypothesis class $\HypClass$ is usually regarded simply as a learner which happens to always produce hypotheses in the class $\HypClass$. This is a key notion, about which we will want to reason computably.

In our setting, to study the computability of proper learning, it will be valuable to consider the case in which the elements of $\HypClass$ are identified by indices bearing additional structure, and thus to consider learners that identify hypotheses in $\HypClass$ by such indices, using a presentation $\PresentationHypClass$. Consequently, and in contrast to the classical setting, we take proper learners to be slightly different objects than ordinary learners. Our proper learners map samples to indices, rather than map samples and features to labels. We then can define a computable proper learner to be simply a proper learner that is computable (similarly to \cref{Defi:Computable learner} of a computable learner).

\begin{definition}\label{Defi:Computable proper learner}
Let $\PresentationHypClass \colon \cI \times \cX \to \cY$ be a presentation of a hypothesis class. A \defn{proper learner} for $\PresentationHypClass$ is a map $\fA\colon(\cX  \times \cY)^{<\omega} \to \cI$.
If the map $A$ defined by $A((x_i, y_i)_{i\in [n]}, x) = \PresentationHypClass(\fA(x_i, y_i)_{i\in [n]}, x)$ is a PAC learner for $\UnderlyingHypClass$, then $\fA$ is a \defn{proper PAC learner} for $\PresentationHypClass$, and we call $A$ the \defn{learner induced} by $\fA$ (as a proper learner for $\PresentationHypClass$). 
If $\PresentationHypClass$ is a computable presentation, we say that a proper learner $\fA$ for $\PresentationHypClass$ is \defn{computable} when it is computable as a map of computable extended metric spaces. 
\end{definition}

Note that the learner $A$ induced by a computable proper PAC learner for $\PresentationHypClass$ in \cref{Defi:Computable proper learner} is a computable learner for $\UnderlyingHypClass$, as we have required both $\fA$ and $\PresentationHypClass$ to be computable. Intuitively, $\PresentationHypClass$ is computably properly PAC learnable if there is a computable function which takes in finite sequences of elements of $\cX \times \cY$ and outputs the index of an element of $\PresentationHypClass$, and where the corresponding learner PAC learns $\UnderlyingHypClass$.

\begin{definition}
Given a hypothesis class $\HypClass$, define $\Partial_{\HypClass}\subseteq (\cX \times \cY)^{<\omega}$ to be the set of those finite sequences $\big((x_1, y_1), \ldots, (x_n, y_n)\big)$ for which
$\big\{(x_1, y_1), \ldots, (x_n, y_n)\big\}$ is a subset of the graph of $h$ for some $h \in \HypClass$, i.e., $ \bigcup_{h \in \HypClass} \coprod_{n \in \omega} \big\{(x, h(x))\st x \in \cX \big\}^n$. 
\end{definition}

Recall that the realizable case restricts attention to measures $\cD$ for which $\cD$-i.i.d.\ sequences are almost surely in the graph of some element of $\HypClass$. In particular, for any such $\cD$ and $n\in\N$, the product measure $\cD^n$ is concentrated on $\Partial_\HypClass \cap (\cX \times \cY)^n$. Note, however, that $\Partial_{\HypClass}$ itself will not in general be Borel, even when $\HypClass$ is. Yet, in the following definition, $\Partial_{\HypClass}$ plays only the role of a subdomain on which the computability of learners in the realizable case is considered, and thus its measure-theoretic properties are of no consequence.

\begin{definition}
Let $\HypClass$ be a hypothesis class. Then a learner $A$ for $\HypClass$ in the realizable case is \defn{computable in the realizable case} for $\HypClass$ if it is computable on $\Partial_\HypClass \times \cX$ as a function between metric spaces  $(\cX \times \cY)^{<\omega} \times \cX$ and $\cY$. 
A proper learner $\fA$ for a computable presentation $\PresentationHypClass$ of $\HypClass$ is \defn{computable in the realizable case} if $\fA$ is computable on $\Partial_\HypClass$ as a function between metric spaces  $(\cX \times \cY)^{<\omega}$ and $\cI$.
\end{definition}

Note that it is possible to have a noncomputable learner for $\HypClass$ which is nevertheless computable in the realizable case for $\HypClass$.  However, all computable learners for $\HypClass$ are computable in the realizable case for $\HypClass$.

It will be important to impose computability constraints on sample functions as well as learners. 

\begin{definition}
A sample function
$m\colon (0, 1)^2 \to \N$ is \defn{computable} if uniformly in $n \in \N$ there are computable sequences
of rationals $(\ell_{n,i})_{i\in\N}$, $(r_{n,i})_{i\in\N}$, $(t_{n,i})_{i\in\N}$, and  $(b_{n,i})_{i\in\N}$
such that
\begin{itemize}
\item 
$U_n  \subseteq m^{-1}(n)$ for every $n\in\N$, and

\item the closure of the  set $\bigcup_{n  \in \N} U_n$ is $(0, 1)^2$,
\end{itemize}
where for each $n$ we define $U_n  = \bigcup_{i \in \N} (\ell_{n,i}, r_{n,i})  \times  (t_{n,i}, b_{n,i})$.
\end{definition}

Given a computable PAC learner and a computable sample function for this learner, one can produce an algorithm that, given an error rate and failure probability, outputs a hypothesis having at most that error rate with at most the stated failure probability. 
If the computable learner is an ERM, then by Theorem~\ref{thm:ERM sample function} it has a computable sample function, and so one obtains such an algorithm.
On the other hand, we will see in Theorem~\ref{bad-sample-function} that not every computable PAC learner (for a given hypothesis class $\HypClass$ and class of distributions $\DD$) admits a computable sample function (with respect to $\HypClass$ and $\DD$).

\subsection{Countable hypothesis classes}

Suppose that $\cX$ is countable and discrete. Requiring that a learner $A$ be computable is then tantamount to asking that the maps $x \mapsto A(S, x)$ be uniformly computable as $S$ ranges over $(\cX \times \cY)^{<\omega}$.
By collecting up this data, such a computable learner $A$ can be encoded as a computable map from $\N$ to $\N$.
In a similar fashion, a computable presentation of a hypothesis class could be encoded by a single computable map from $\N$ to $\N$.

The paper \cite{ALT} studies computable PAC learning in the setting where $\cX = \N$, a countable discrete metric space. As such, they are able to work with the encodings of these simplified notions of computable learners and presentations of hypothesis classes, as we have just sketched.

\subsection{Examples}

To illustrate these definitions, we now describe two examples --- one a very basic one in this formalism, and the other a standard example from learning theory. 

\subsubsection{``Apply'' function}
Let the index space $\cI$ be $2^\N$ and the sample space $\cX$ be $\N$.  
We define the ``apply’’ presentation of the hypothesis class $2^\N$ to be the map $\PresentationHypClass\colon \cI  \times \cX \to \{0, 1\}$ where $\PresentationHypClass(x, n) = x(n)$. Note that while $\PresentationHypClass$ is computable, there is no single Turing degree which bounds every hypothesis in $\UnderlyingHypClass = 2^\N$. In particular, this example demonstrates that the notion of computable hypothesis class that we consider is fundamentally more general than the corresponding notion in \cite{ALT}, which considers only countable collections of hypotheses.

\subsubsection{Decision stump}
Recall the decision stump problem from classical learning theory: $\cX = \R$, $\cY = \{0, 1\}$, and $\HypClass = \{\Indicator{>c} \st c \in \R\}$. In the realizable case, the learning problem amounts to estimating the true cutoff point $c$ from a sample $S = (x_i, y_i)_{i \in [n]}$ for which $y_i = 1$ if and only if $x_i > c$. It is well-known to be PAC learnable in the realizable case via the following algorithm:
\begin{itemize}
  \item[1.] If $S$ has negatively labeled examples (i.e., $(x_i, y_i)$ with $y_i = 0$), then set $m$ to be the maximal such $x_i$. Otherwise, set $m$ to be the minimal feature among positively labeled examples.
  \item[2.] Return $\Indicator{>m}$.  
\end{itemize}
In particular, this implements an ERM learner for $\HypClass$ in the realizable case.
Further, as $\HypClass$ has VC dimension $1$, it is a PAC learner for $\HypClass$ in the realizable case by the equivalence of clauses 1 and 4 in Theorem~\ref{thm:fundamental}.

The classical algorithm does not give rise to a computable learner in the sense of \cref{def:learner}, however, as $\Indicator{>m}$ cannot be computed from $S$. In particular, undecidability of equality for real numbers obstructs such a computation from being performed over $\R$. In order to more sensibly cast the problem in a computable setting, we restrict focus to cutoff points located at computable reals and take the noncomputable reals as the domain set $\cX$. 

Now consider the computable presentation 
  $\Hstep \colon \R_c \times (\R \setminus \R_c) \rightarrow \{0, 1\}$
of a hypothesis class with index set the computable reals $\R_c$,
given by
  $\Hstep(c, x ) = \Indicator{>c}(x)$.
Its underlying hypothesis class $\HstepUnderlying = \{\Indicator{>c} \st c \in \R_c\}$ consists of computable functions (whose domains are $\R \setminus \R_c$), thus proper learners have a chance of success. Nevertheless, the classical algorithm fails: $m$ will reside in $\cX$, and thus $\Indicator{>m}$ will be noncomputable as a function on $\cX$ (even when one has access to $m$). We will exhibit a proper learner $\fAstep$ for $\Hstep$ that is computable in the realizable case and whose induced learner is an ERM.

Fix a computable enumeration $(q_i)_{i\in\N}$ of $\Q$ and uniformly enumerate a computable presentation of each as a computable real. 

\begin{algorithm}[Algorithm $\fAstep$]
Given a sample $S = (x_i, y_i)_{i \in [n]}$,
output the least $i\in\N$ for which the empirical error of $\Indicator{>q_i}$ is 0.
\end{algorithm}

\begin{proposition}\label{Stump algo properties}
$\fAstep$ is a proper learner for $\Hstep$ that is computable in the realizable case and whose induced learner is an ERM.
\end{proposition}
\begin{proof}
Observe that the sequence of functions
$(\Indicator{>q_i})_{i\in\N}$ is uniformly computable on $\cX = \R \setminus \R_c$. 
The empirical error of each $\Indicator{>q_i}$ can be computed exactly on any sample (and hence compared with 0). The loop terminates upon reaching a rational $q_i$ that separates the sample $S$, one of which must exist for any $S$ under consideration in the realizable case. 
\end{proof}

\begin{corollary}\label{cor:stump learnable}
$\fAstep$ is a computable proper PAC learner in the realizable case for $\Hstep$.
\end{corollary}
\begin{proof}
By \cref{Stump algo properties},  $\fAstep$ is a computable proper learner in the realizable case for $\Hstep$, whose induced learner is an ERM.
The class $\HstepUnderlying$ has VC dimension 1, and so by the equivalence of clauses 1 and 4 in Theorem~\ref{thm:fundamental}, the learner induced by $\fAstep$ is a PAC learner in the realizable case.
\end{proof}

In fact, we will see shortly in Theorem~\ref{thm:realizable ERM} that \cref{cor:stump learnable} is an instance of a more general result, namely that all classes with computable presentations have computable ERM learners in the realizable case.

\subsection{Computable learners with noncomputable sample functions}

\cref{bad-sample-function} shows that even when a hypothesis class $\HypClass$ and class of distributions $\DD$ admit some computable PAC learner with a computable sample function, not all computable learners for $\HypClass$ with respect to $\DD$ must have a computable sample function.

Therefore, when investigating the computability of algorithms for outputting a hypothesis (with the desired error rate and failure probability), we must consider the computability of a pair consisting of a PAC learner and sample function, not merely the PAC learner alone.

The intuition behind the proof of \cref{bad-sample-function} is that we can enumerate those programs that halt, and whenever the $n$th program to halt does so, we then coarsen all samples of size $n$ up to accuracy $2^{-s}$, where $s$ is the size of the program. Consequently, for each desired degree of accuracy, we eventually obtain answers that are never coarsened beyond that accuracy. On the other hand, knowing how many samples are needed for a given accuracy allows us to determine a point past which we never again coarsen to a given level. This then lets us deduce when a given initial segment of the halting set has stabilized. 

\begin{definition}
For $M \in \N$, let $\DD_M$ be the collection of Borel probability distributions $\cD$ over $(\R \setminus \R_c) \times \{0, 1\}$ such that
\begin{itemize}
  \item[(i)] $L_{\cD}(h) = 0$ for some element $h$ of $\Hstep$, and 
  \item[(ii)] $\cD$ is absolutely continuous (with respect to Lebesgue measure) and has a probability density function bounded by $M$. 
\end{itemize}
\end{definition}

\begin{theorem}
\label{bad-sample-function}
For each $M \in \N$, there is a learner $A$ on $\cX = \R \setminus \R_c$ and $\cY = \{0, 1\}$ such that
\begin{itemize}
  \item $A$ is computable in the realizable case with respect to $\HstepUnderlying$,
  \item $A$ is a PAC learner for $\HstepUnderlying$ over $\DD_M$, and 
  \item $\HaltingSet$ is computable from any sample function for $A$ (as a learner for $\HstepUnderlying$ over $\DD_M$).
\end{itemize}
\end{theorem}

\begin{proof}
Define the function $\alpha \colon \Q \times \N \to \Q$ by  $\alpha(q, \ell) = \floor{2^\ell q}/2^\ell$, and let $c \colon (\cX \times \cY)^{<\omega} \to \Q$ be such that
$c\big(\{(x_i, y_i)\}_{i \in [k]}\big)$ is the rational $q$ of least index attaining zero empirical error on $\{(x_i, y_i)\}_{i \in [k]}$ if one exists, and 0 otherwise. 
Hereafter, we will additionally demand that the computable enumeration of $\Q$ employed by $c$ be one which enumerates \sfrac{1}{3} first.  Define $c^*\colon (\cX \times \cY)^{<\omega} \times \N \to \Q$ by
$c^*(S, n) = \alpha(c(S), n)$, i.e., the previous decision stump learner discretized to accuracy $2^{-n}$.

Let $(e_k)_{k \in \N}$ be a computable enumeration without repetition of all $e\in\N$ for which $\{e\}(0)\halts$.
For $S \in (\cX \times \cY)^{<\omega}$, write $\len(S)$ for its length.
Define the function $A\colon (\cX \times \cY)^{<\omega} \times \cX \to \cY$ by
$A(S, x) = h_{c^*(S, e_{\len(S)})}(x)$.
In other words, we discretize the decision stump algorithm to accuracy $2^{-e_{\len(S)}}$.
Note that because $\lim_k e_k = \infty$, 
we can find arbitrarily good approximations
as we increase the sample size,
even if (as we will show) we cannot compute how large such samples must be.

Note that $A$ is computable in the realizable case. Further, for every $r\in\N$ there is an $i\in\N$ such that 
$e_{r^*} > i$ for all $r^* \geq r$. 
Then for every integer $\ell > 0$, there is an $n\in\N$ such that whenever $\len(S) > n$, the set $U = \{x \st \Hstep(\fAstep(S), x) \neq A(S, x)\}$  is contained in an interval of length $2^{-\ell}$. $A$ is thus a PAC learner for $\HstepUnderlying$ over $\DD_M$, as the loss incurred by $A$ on $U$ is bounded uniformly over $\DD_M$ by $2^{-\ell} \cdot M$. 

Let $m(\epsilon, \delta)$ be a sample function for $A$ and consider $n \in \N$. We will compute the function $\HaltingSet$ restricted to the set $[n] = \{0, \dots, n-1\}$. 
Fix any rational $\delta \in (0, 1)$, and set $m_n = m(2^{-(n + 2)}, \delta)$.
Suppose there is some $i > m_n$ such that $e_i < n$. Then 
given a sample $S$ of size $i$, the function $A(S, \pars)$
will discretize $c(S)$ to an accuracy below $2^{-n}$.
This would cause $A$ to incur a true loss of at least $2^{-n}$ on the distribution which is uniform on features in $[0, 1]$ and takes labels according to $\Indicator{>\sfrac{1}{3}}$, as $\alpha(\sfrac{1}{3}, k) \leq \sfrac{1}{3} - 2^{-(k+2)}$,
a contradiction.
Hence $i \le m_n$ whenever $e_i < n$.
We can therefore determine membership in $\{e_k \st k \in \N\} \cap [n]$, and hence can compute
$\HaltingSet$ restricted to $[n]$.
\end{proof}

\pagebreak
\section{Computability of learners}\label{sec:bounds}

We now turn to the question of how computable a learner can be, for a hypothesis class with a computable presentation.

Throughout this section, we remain in the setting of binary classification, i.e., $\cY = \{0, 1\}$. 

\subsection{Upper bounds}
\label{subsec:upper}

For any computable presentation of a hypothesis class, 
we establish a concrete upper bound on the complexity of some ERM, which depends only on the index space.

\begin{theorem}
\label{thm:agnostic ERM}
Suppose $\PresentationHypClass\colon \cI \times \cX \to \cY$ is a computable presentation of a hypothesis class. Then there is an ERM for $\UnderlyingHypClass$ that is strongly Weihrauch reducible to $\lim_{\cI}$.
\end{theorem} 

\begin{proof}
Fix a sample $S = (x_i, y_i)_{i \in [n]}$. To invoke $\lim_\cI$, we introduce a procedure for approximating an input $z = (z_i)_{i \in \N} \in \cI^\N$. In particular, we approximate $z$ using the sequence $(z_k)_{k \in \N}$, with each $z_k \in \cI^\N$ taking the form $z_k = (z_k^1, \dots, z_k^{k-1}, z_k^k, z_k^k, \dots)$, i.e., constant after the $(k-1)$th term. 

$z_k^j$ is computed as follows, for $j \in [k]$:
\begin{itemize}

	\item[1.] Take balls around the $x_i$ and around the first $j$ ideal points of $\cI$, all of radius $2^{-k}$. In addition, calculate which value is taken by $y_i \in \{0,1\}$. 
	\item[2.] For each of the first $j$ ideal points of $\cI$, use $\PresentationHypClass$ to determine whether the balls around the $x_i$ and the ideal point suffice to calculate a well-defined empirical error with respect to $S$. 
	\item[3.] If none of the first $j$ ideal points induce a well-defined empirical error, set $z_k^j$ to be the first ideal point of $\cI$. Otherwise, set $z_k^j$ to be the first ideal point which attains minimal empirical error among the first $j$ ideal points.

\end{itemize}

As $\PresentationHypClass$ is continuous, and as there are only finitely many possible empirical errors, if $w \in \cI$ is such that $\curry{\PresentationHypClass}(w)$ has minimal empirical error with respect to $S$, then there must be an open ball around $w$ where all elements of the ball give rise to a function with the same minimal empirical error (with respect to $S$). In particular, there must be an ideal point $c$ such that $\curry{\PresentationHypClass}(c)$ has minimal empirical error with respect to $S$. Therefore $z = (z_j^j)_{j \in \N}$ converges to the ideal point with minimal index among those that give rise to minimal empirical error with respect to $S$.
Calling $\lim_\cI$ on $z$ thus is a proper learner whose induced learner for $\UnderlyingHypClass$ is an ERM, as desired.
\end{proof}

Furthermore, in this setting there is always an ERM that is computable in the realizable case.
This result can be viewed as a generalization of \cite[Theorem~10]{ALT}.

\begin{theorem}\label{thm:realizable ERM}
Suppose $\PresentationHypClass \colon \cI \times \cX \to \cY$ is a computable presentation of a hypothesis class. Then there is an ERM for $\UnderlyingHypClass$ that is computable in the realizable case. 
\end{theorem}
\begin{proof}
We will construct a proper learner whose induced learner for $\UnderlyingHypClass$ is an ERM that is computable in the realizable case.
Suppose $S\in \Partial_{\UnderlyingHypClass}$.
There is some $w \in \cI$ such that $\curry{\PresentationHypClass}(w)$ has empirical error $0$ with respect to $S$. Because there are only finitely many possible values of the empirical error with respect to $S$, there must be some open ball $B$ around $w$  such that for all elements $w^* \in B$, the function  $\curry{\PresentationHypClass}(w^*)$
has empirical error $0$ with respect to $S$. In particular, there must be some ideal point $c$ in this ball. Therefore the algorithm which searches through all ideal points and returns the first to attain an empirical error $0$ with respect to $S$ will eventually halt.
\end{proof}

\begin{remark}
The algorithms in Theorems~\ref{thm:agnostic ERM} and \ref{thm:realizable ERM} would have failed had $\cY$ not been computably discrete, in which case verifying that a hypothesis incurs an empirical error of 0 would not be computable. When $\cY = \{0, 1\}$, as in this paper, the predictions of hypotheses on features can be deduced exactly, allowing for precise computation of empirical errors. If $\cY = \R$, in contrast, then predictions of hypotheses $h$ take the form $(q_k - 2^{-k}, q_k + 2^{-k})$ for $q_k \in \Q$ and chosen $k \in \N$, amounting to the information that $h(x) \in (q_k - 2^{-k}, q_k + 2^{-k})$.

Some such intervals allow one to conclude that $h(x) \neq y$, namely when $y \notin (q_k - 2^{-k}, q_k + 2^{-k})$, and thus that $h$ does not attain an empirical error of 0. Yet no such interval allows one to conclude that $h(x) = y$ for even a single example $(x, y)$ if $y$ may take any real value, much less that $h$ attains an empirical error of 0 across an entire sample.
\end{remark}

The restricted setting of computability in the realizable case, as in Theorem~\ref{thm:realizable ERM}, provides a stopping criterion for detecting a hypothesis in $\UnderlyingHypClass$ attaining minimal empirical risk on $S$, 
thereby eliminating the need for $\lim_{\cI}$. A similar criterion would arise if the size of the restriction of a (computably presented) class $\HypClass$ to a given sample $S$ could be known in advance. In such a case, one could walk through the ideal points of $\cI$ as in \ref{thm:realizable ERM} until all such behaviors on $S$ are encountered, subsequently returning one which attains the minimal empirical error.

\begin{theorem}\label{thm:computable behaviors}
Suppose $\PresentationHypClass \colon \cI \times \cX \to \cY$ is a computable presentation of a hypothesis class, and that for all finite $U \subseteq \cX$, the size of $\{h \restricted{U} \st h \in \UnderlyingHypClass\}$ can be computed, uniformly in $U$. Then an ERM learner for $\UnderlyingHypClass$ is computable. 
\end{theorem}
\begin{proof}
Let $n \in\N$.
Define $\PresentationHypClass^n\colon\cI \times \cX^n \to \cY^n$ to be the map where $\PresentationHypClass^n(w, (x_j)_{j \in [n]}) = (\PresentationHypClass(w, x_j))_{j \in [n]}$.
Note that this is a continuous function, and hence for all
$w \in \cI$ and for every $u \in \cX^n$ there is an ideal point $c$ such that $\PresentationHypClass^n(w, u) = \PresentationHypClass^n(c, u)$.

Suppose $U\subseteq \cX$ is finite.
We then have
\[
\bigl|\{h\restricted{U} \st h \in \UnderlyingHypClass)\}\bigr|
=
\bigl|\bigl\{\curry{\PresentationHypClass}(c)\restricted{U} \st c\text{ is an ideal point of } \cI\bigr\}\bigr|.
\]
In particular, by searching through all the ideal points of $\cI$  we realize all behavior (restricted to $U$) that occurs in $\curry{\PresentationHypClass}$. So, from $\bigl|\{h\restricted{U} \st h \in \UnderlyingHypClass\}\bigr|$ we can compute ideal points of $\cI$ realizing all such behavior. From this it is straightforward to choose an ideal point which minimizes the empirical error on any sample $(x_i, y_i)_{i \in  [m]}$ where $\{x_i \st i \in [m]\} = U$.
\end{proof}

It has been shown in \cite{FloydWarmuth} that the computability condition of Theorem~\ref{thm:computable behaviors} is enjoyed by maximum classes, i.e., those which achieve the bound of the Sauer--Shelah lemma. We can thus conclude computable PAC learnability for such maximum classes.

\begin{corollary}
If $\PresentationHypClass \colon \cI \times \cX \to \cY$ is a computable presentation of a hypothesis class, and $\UnderlyingHypClass$ is a maximum class of finite VC dimension, then it is computably PAC learnable.
\end{corollary}

\subsection{Lower bounds}
\label{subsec:lower}

We now show that in appropriate circumstances, all proper learners must have a certain complexity, thereby providing some corresponding lower bounds.

\subsubsection{Discrete index spaces}

Suppose the index space of a computable presentation of a hypothesis class is infinite and discrete (hence isomorphic to $\N$). \cref{thm:agnostic ERM} shows that there is an ERM that is strongly Weihrauch reducible to $\lim_\N$.  

We now provide a partial converse in the same setting,
showing that there is a computable presentation of a hypothesis class with discrete index space such that
$\HaltingSet$ is strongly Weihrauch reducible to any proper PAC learner for the presentation along with any sample function for the proper learner.
In particular, for this presentation, there is no computable procedure for outputting hypotheses from samples in a manner that PAC learns the underlying hypothesis class.

The hypothesis class that we will use in the proof of \cref{thm:discrete-lower} is similar to that used to prove \cite[Theorem~11]{ALT}.

\begin{theorem}\label{thm:discrete-lower}
There is a hypothesis class that is PAC learnable but admits a computable presentation $\PresentationHypClass$ with discrete index space such that $\HaltingSet  \lesW (\fA, m)$ whenever $\fA$ is a proper PAC learner for $\PresentationHypClass$ and $m$ is a sample function for the learner induced by $\fA$ for $\UnderlyingHypClass$.
\end{theorem}

\begin{proof}
Given an enumeration $Z = (z_i)_{i \in \N}$ without repetition of some subset of $\N$, define $\cD_Z$ to be the computable metric space with underlying set $\{z_i \st i\in\N\}$, the discrete metric taking distances $\{0, 1\}$, and sequence of ideal points $Z$.

Let $E = (e_i)_{i \in \N}$ be a computable enumeration without repetition of all natural numbers $e$ such that $\{e\}(0)\halts$.
Given two natural numbers $n_0$ and $n_1$, we write $n_0 \sim n_1$ if either (a) both $\{n_0\}(0)\nohalts$ and $\{n_1\}(0)\nohalts$, or (b) both $\{n_0\}(0)\halts$ and $\{n_1\}(0)\halts$ and the programs $n_0$ and $n_1$ take the same number of steps to halt on input $0$.

Let the index space $\cI$ be $\cD_E$, as defined above. Let the sample space $\cX$ be $\cD_{(i)_{i\in\N}}$, and  let $\PresentationHypClass\colon \cI \times \cX \to \{0, 1\}$ be the map where $\PresentationHypClass(n_0, n_1) = 1$ if and only if $n_0 \sim n_1$. Note that $\PresentationHypClass$ is computable via the following algorithm:
\begin{enumerate}
\item Run program $n_0$ on input $0$ until it halts, which it must as $n_0  \in \cI$. Let $k$ be the number of steps it took to halt.
\item Run program $n_1$ on input $0$ for $k$ steps.
\item If $n_1$ takes precisely $k$ steps to halt on input $0$ then return $1$, and otherwise (i.e., if it takes fewer steps or has not yet halted) return $0$. 
\end{enumerate}

First we show that $\UnderlyingHypClass$ shatters no set of size $2$, so that it has VC dimension $1$ and hence is PAC learnable by \cref{thm:fundamental}. Let $n_0, n_1 \in \N$ be distinct.
If there exists an $h \in \UnderlyingHypClass$ with $h(n_0) = 1$ and $h(n_1) = 0$, then there is some $k$ such that $\{n_0\}(0)$ halts in exactly $k$ steps but $\{n_1\}(0)$ does not. But then there is no $g \in \UnderlyingHypClass$ with $g(n_0) = 1$ and $g(n_1) = 1$. 
Therefore $\UnderlyingHypClass$ does not shatter the set $\{n_0, n_1\}$.

Now suppose that $\fA$ is a proper PAC learner for $\PresentationHypClass$, and let $m$ be a sample function for the induced PAC learner for $\UnderlyingHypClass$. We will show that $\HaltingSet \lesW (\fA, m)$.

Let $M = m(\epsilon, \delta)$ for any choice of $\epsilon, \delta \in (0, 1)$. 
Given $n \in \N$, let $S_n = \bigl((n, 1)\bigr)_{i \in [M]}$, i.e., $M$ copies of $(n, 1)$.  Let $z_n  = \fA(S_n)$ and let $\mu_n$ be the measure with a single point mass on $(n, 1)$. Note that $\mu_n^M$ places full measure on $S_n$ and that for any map $h\colon \cX \to \{0, 1\}$, its loss with respect to $\mu_n$ is either $0$ or $1$. 

Note that if $\{n\}(0)\halts$ then there is an $n^* \in \cI$ (namely, $n^*  = n$) such that the minimum loss with respect to $\mu$ is $0$. As $\fA$ is a proper PAC learner for $\PresentationHypClass$, we must then have $z_n \sim n$. Otherwise, $\fA$ incurs an error of $1>\epsilon$ with probability $1>(1-\delta)$ over $\mu_n$ on samples of size $M$, producing contradiction with the PAC condition on $m(\epsilon, \delta)$.
On the other hand, if $\{n\}(0)\nohalts$, then for any $e \in \cI$ we must have $e \not \sim n$, and in particular $z_n \not \sim n$. Therefore $n \mapsto 1 - z_n$ is precisely the function $\HaltingSet$. In particular, this shows that $\HaltingSet \lesW (\fA, m)$.
\end{proof}

\subsubsection{Rich index spaces}

When the index space $\cI$ of a computable presentation of a hypothesis class is rich, we have $\lim_{\cI} \eqsW \lim_{\Baire}$.
In this case, \cref{thm:agnostic ERM} shows that there is an ERM that is strongly Weihrauch reducible to $\lim_{\Baire}$.

We also provide a partial converse in this situation, using the notion of parallelization. 
We show that there is a computable presentation of a hypothesis class with rich index space such that
$\lim_{\Baire}$ is strongly Weihrauch reducible to the parallelization of any proper PAC learner for the presentation along with any sample function for the proper learner.

\begin{remark}
\label{rem:parallel explanation}
It is worth taking a moment to discuss why, when considering learners on continuum-sized metric spaces, we study the parallelization of the learner as opposed to the learner itself. When comparing the relative computational strength of two maps $f$ and $g$, the notion of $g$ being ``more complex'' than $f$ can be intuitively thought of as the statement that one can compute $f$ when given access to $g$. This is made precise using the formalism of strong Weihrauch reducibility, in which a single application of $f$ must be computed using a single application of $g$ (possibly along with some uniform pre- and post-processing).

However, the manner in which we are discussing learners, namely as maps from $(\cX \times \cY)^{<\omega} \times \cX$ to $\{0, 1\}$ (as opposed to maps from $(\cX \times \cY)^{<\omega}$ to $\{0, 1\}^{\cX}$), means that a single application of a learner can only return a single bit of information about its input. In contrast, $\lim_{\Baire}$ is a map from $\Baire$ to $\Baire$ for which a single application contains countably many bits of information. As such, our representation of learners is not well suited to be compared to $\lim_{\Baire}$ if we (somewhat artificially) allow only a single application of the learner. 
We can overcome this obstacle by instead considering the parallelization of the learner, i.e., by allowing ourselves to simultaneously ask countably many questions of the learner, rather than a single one. 
This is what we do in \cref{rich lower bound} (in the analogous setting for proper learners).

By \cref{lem:reducible to parallelization}, any function is strongly Weihrauch reducible to its parallelization, and  
by \cref{lem:limBaire is parallelizable}, 
$\lim_{\Baire}$ is strongly Weihrauch equivalent to its own parallelization.
Hence not much is lost when establishing that $\lim_{\Baire}$ is a lower bound on the parallelization of a proper learner (as opposed to a lower bound on the proper learner itself).
\end{remark}

\begin{theorem}
\label{rich lower bound}
There is a hypothesis class that is PAC learnable but which admits a computable presentation $\PresentationHypClass$ 
such that $\lim_{\Baire} \lesW (\parallelization{\fA}, m)$ whenever $\fA$ is a proper PAC learner for $\PresentationHypClass$ and $m$ is a sample function for the PAC learner
for $\UnderlyingHypClass$ that $\fA$ induces.
\end{theorem}

\begin{proof}
Let $\cX$ be the product of computable metric spaces $\N$ and $\Baire$,
 and let the index space $\cI$ be 
$\{(e, z) \in \N \times \Baire \st \{e\}^z(0)\halts\}$ with distance inherited from $\cX$ and ideal points of the form $(e, z)$ where $z$ has only finitely many nonzero values, ordered by when the respective programs with oracles halt on input $0$.

Given $(e_0, z_0), (e_1, z_1) \in \cX$, we write $(e_0, z_0) \sim (e_1, z_1)$ when (a) $e_0 = e_1$ and  (b) $\{e_0\}^{z_0}(0)\halts$ if and only if $\{e_1\}^{z_1}(0)\halts$, with program $e_0$ with oracle $z_0$ taking the same number of steps to halt on input $0$ as does $e_1$ with oracle $z_1$ (when they both halt).
Define $\PresentationHypClass \colon \cI \times \cX \to \{0, 1\}$  by
\[
\PresentationHypClass\big((e_0, z_0), (e_1, z_1)  \big) =
\begin{cases}
1 & (e_0, z_0) \sim (e_1, z_1); \\
0 & \text{otherwise}.
\end{cases}
\] 
Note that  $\PresentationHypClass$ is computable because $\{e_0\}^{z_0}(0)\halts$ for every $(e_0, z_0) \in \cI$.

First we show that $\UnderlyingHypClass$ shatters no set of size $2$, so that it has VC dimension $1$ and hence is PAC learnable by \cref{thm:fundamental}. Let $(e_0, z_0), (e_1, z_1) \in \cX$ be distinct.
If there exists an $h \in \UnderlyingHypClass$ with $h(e_0, z_0) = 1$ and $h(e_1, z_1) = 0$, then there is some $k$ such that the program $e_0$ with oracle $z_0$ halts on input $0$ in exactly $k$ steps but either $e_0 \neq e_1$ or the program $e_1$ with oracle $z_1$ does not halt on input $0$ in exactly $k$ steps.
But then there is no $g \in \UnderlyingHypClass$ with $g(e_0, z_0) = 1$ and $g(e_1, z_1) = 1$.
Therefore $\UnderlyingHypClass$ does not shatter the set $\{(e_0, z_0), (e_1, z_1)\}$.

Now suppose that $\fA\colon (\cX \times \cY)^{<\omega} \to \cI$ is a proper PAC learner for $\PresentationHypClass$, let $A$ be the induced PAC learner for $\UnderlyingHypClass$, and
let $m$ be a sample function for $A$ (as a PAC learner for  $\UnderlyingHypClass$).
We will show that $\J \lesW (\parallelization{\fA}, m)$. Then by \cref{limBaire eqsW J}, we will have
$\lim_{\Baire} \lesW (\parallelization{\fA}, m)$.

Let $z \in \Baire$.  We aim to uniformly compute $z'$ using $\fA$, $m$, and $z$.
First we preprocess. Calculate $k = m(\epsilon, \delta)$ for any choice of $\epsilon, \delta \in (0, 1)$ and construct the sequence $S_{e, z} = \big( ((e, z), 1)^k \big)_{e \in \N}$. Then, apply $\parallelization{A}$ to obtain a sequence $(\ell_e, s_e)_{e \in \N}$.

Now consider the measure $D_{(e, z)}$ which places a pointmass on $((e, z), 1)$. Because $A$ is a PAC learner, we have
\[
\Pr_{S \sim D^k_{(e, z)}} \Bigl ( \bigl|L_{D_{(e,z)}} (A(S)) - \min_{w \in \cI} L_{D_{(e, z)}}
( \curry{\PresentationHypClass}(w)) \bigr| < \epsilon \Bigr) > 1- \delta.
\]
Therefore, as $D_{(e, z)}$ is a pointmass, we have
$
\bigl|L_{D_{(e,z)}}(A(S_{e,z})) - \min_{w \in \cI} L_{D_{(e, z)}}(\curry{\PresentationHypClass}(w))\bigr| < \epsilon$.
Again because $A$ is a PAC learner and $D_{(e, z)}$ is atomic, we have an equivalence between the following statements:
\begin{enumerate}
  \item $A(S_{e, z})(e, z) = 1$.
  \item $L_{D_{(e, z)}}(A(S_{e, z})) = 0$.
  \item $L_{D_{(e, z)}}(\curry{\PresentationHypClass}(w)) = 0$ for some $w \in \cI$. 
  \item $\PresentationHypClass(w, (e, z)) = 1$ for some $w \in \cI$.
\end{enumerate}
In particular, (3) $\Rightarrow$ (2) because $D_{(e, z)}^k$ concentrates mass on $S_{(e, z)}$, so otherwise $A$ would be guaranteed to incur a loss of $1 > \epsilon$ when trained on samples drawn from $D_{(e, z)}^k$, contradicting the PAC condition on $m(\epsilon, \delta)$. 

Now note that if $\{e\}^{z}(0)\halts$, then there is a $w = (e, z)\in \cI$ such that $\PresentationHypClass(w, (e, z)) = 1$; by the previous equivalence, this implies that $A(S_{e, z})(e, z) = 1$. 

We are now equipped to post-process and calculate $z'(n)$. 
If $n \neq \ell_n$,
then $A(S_{n, z})(n, z) = 0$ and, via $\lnot(1) \Rightarrow \lnot(4)$ in the equivalence, 
$\{n\}^z(0)\nohalts$, meaning $z'(n) = 0$. 

Otherwise, $n = \ell_n$. First compute
$\{n\}^{s_n}(0)$. This computation is 
guaranteed to halt, by definition of $\cI$ and the fact that $\fA$ is a proper learner.  Let $t$ be the number of steps it took to halt.
Next run $\{n\}^z$ on input $0$ for $t$ steps. If it halts within $t$ steps, then 
$\{n\}^z(0)\halts$ and so $z'(n) = 1$.
If $\{n\}^z$ has not halted on input $0$ within $t$ steps, then $A(S_{n, z})(n, z) = 0$ and the equivalence again implies that 
$\{n\}^z(0)\nohalts$, meaning $z'(n) = 0$. 
\end{proof}

\section*{Acknowledgements}
The authors would like to thank Caleb Miller for valuable discussion on the topic, particularly in helping refine the notion of computable PAC learning and in describing the computable algorithm for the decision stump.

An extended abstract \cite{CCA-extended-abstract} announcing related results in a different setting was presented at the Eighteenth International Conference on Computability and Complexity in Analysis (July 26--28, 2021).

This material is based upon work supported by the National Science Foundation under grant no.\ CCF-2106659. Freer's work is funded in part by financial support from the Intel Probabilistic Computing Center.


\newcommand{\etalchar}[1]{$^{#1}$}
\providecommand{\bysame}{\leavevmode\hbox to3em{\hrulefill}\thinspace}
\providecommand{\MR}{\relax\ifhmode\unskip\space\fi MR }
\providecommand{\MRhref}[2]{%
  \href{http://www.ams.org/mathscinet-getitem?mr=#1}{#2}
}
\providecommand{\href}[2]{#2}


\end{document}

%% file: macros.tex
\newcommand{\defn}[1]{\textbf{#1}}

\newcommand{\Indicator}[1]{\mathbbm{1}_{#1}}
\newcommand{\restricted}[1]{\!\!\upharpoonright_{#1}}
\newcommand{\halts}{\!\downarrow}
\newcommand{\nohalts}{\!\uparrow}

\newcommand{\curry}[1]{\widetilde{#1}}
\newcommand{\parallelization}[1]{\widehat{#1}}

\newcommand{\R}{\mathbb R}

\newcommand{\DD}{\mathbb D}
\newcommand{\Q}{\mathbb Q}

\newcommand{\N}{\mathbb N}
\newcommand{\Sier}{\mathbb S}

\newcommand{\bbX}{\mathbb X}
\newcommand{\bbY}{\mathbb Y}

\newcommand{\fA}{\mathfrak A}
\newcommand{\cD}{\mathcal D}
\newcommand{\cF}{\mathcal F}
\newcommand{\cG}{\mathcal G}

\newcommand{\HypClass}{\mathcal H}
\newcommand{\PresentationHypClass}{\mathfrak H}
\newcommand{\UnderlyingHypClass}{\PresentationHypClass^\dagger}

\newcommand{\cI}{\mathcal I}
\newcommand{\cX}{\mathcal X}
\newcommand{\cY}{\mathcal Y}

\newcommand{\Hstep}{\PresentationHypClass_{\mathrm{step}}}
\newcommand{\HstepUnderlying}{\PresentationHypClass_{\mathrm{step}}^\dagger}
\newcommand{\fAstep}{\mathfrak{A}_{\mathrm{step}}}

\newcommand{\HaltingSet}{\emptyset'}

\newcommand{\st}{\ : \ } 
\newcommand{\pars}{\,\cdot\,} 

\newcommand{\Partial}{\Phi}
\DeclareMathOperator{\argmin}{argmin}
\DeclareMathOperator{\range}{range}
\DeclareMathOperator{\domain}{dom}

\newcommand{\floor}[1]{\lfloor {#1} \rfloor}

\DeclareMathOperator{\len}{len}

\newcommand{\CauchySeq}{\mathrm{CS}}

\newcommand{\Baire}{\N^\N}
\newcommand{\Cantor}{2^\N}

\newcommand{\lesW}{\le_{\mathrm{sW}}}
\newcommand{\eqsW}{\equiv_{\mathrm{sW}}}

\DeclareMathOperator{\J}{J}